\documentclass{article}

\usepackage{amsmath, amsthm, amssymb}
\usepackage[ruled]{algorithm2e} 
\usepackage{algpseudocode}
\usepackage{graphicx}
\usepackage{verbatim}
\usepackage{natbib}
\usepackage{caption}
\usepackage{subcaption}
\usepackage{fancyvrb}
\usepackage{enumerate}
\usepackage{enumitem}
\usepackage{relsize}
\usepackage{hyperref}
\usepackage[margin=1.5in]{geometry}
\hypersetup{colorlinks,citecolor=blue,urlcolor=blue,linkcolor=blue}
\usepackage{diagbox}
\usepackage{stefan_tex}
\usepackage{float}
\usepackage{url}
\usepackage{multirow}
\usepackage{natbib}

\usepackage[utf8]{inputenc}
\usepackage[english]{babel}

\usepackage{accents}
\theoremstyle{definition}

\usepackage{pgfplots}
\usepackage{tikzscale}
\pgfplotsset{compat=newest}
\usepgfplotslibrary{groupplots}
\usepgfplotslibrary{polar}
\usepgfplotslibrary{smithchart}
\usepgfplotslibrary{statistics}
\usepgfplotslibrary{dateplot}
\usepgfplotslibrary{external}
\usetikzlibrary{arrows.meta}
\usetikzlibrary{backgrounds}
\usepgfplotslibrary{patchplots}
\usepgfplotslibrary{fillbetween}
\pgfplotsset{%
layers/standard/.define layer set={%
    background,axis background,axis grid,axis ticks,axis lines,axis tick labels,pre main,main,axis descriptions,axis foreground%
}{grid style= {/pgfplots/on layer=axis grid},%
    tick style= {/pgfplots/on layer=axis ticks},%
    axis line style= {/pgfplots/on layer=axis lines},%
    label style= {/pgfplots/on layer=axis descriptions},%
    legend style= {/pgfplots/on layer=axis descriptions},%
    title style= {/pgfplots/on layer=axis descriptions},%
    colorbar style= {/pgfplots/on layer=axis descriptions},%
    ticklabel style= {/pgfplots/on layer=axis tick labels},%
    axis background@ style={/pgfplots/on layer=axis background},%
    3d box foreground style={/pgfplots/on layer=axis foreground},%
    },
}

\usepackage{soul}

\graphicspath{{./figures/}}

\theoremstyle{plain}
\newtheorem{prop}{Proposition}

\newtheorem{coro}[prop]{Corollary}
\newtheorem{lemm}[prop]{Lemma}
\newtheorem{theo}[prop]{Theorem}

\theoremstyle{definition}

\theoremstyle{remark}

\title{Thompson Sampling with Unrestricted Delays}

\author{
\makebox[45mm]{Han Wu} \\ Stanford University \and
\makebox[45mm]{Stefan Wager} \\  Stanford University}

\begin{document}

\maketitle

\begin{abstract}
We investigate properties of Thompson Sampling in the stochastic multi-armed bandit problem with delayed feedback. In a setting with i.i.d delays, we establish to our knowledge the first regret bounds for Thompson Sampling with arbitrary delay distributions, including ones with unbounded expectation. Our bounds are qualitatively comparable to the best available bounds derived via ad-hoc algorithms, and only depend on delays via selected quantiles of the delay distributions. Furthermore, in extensive  simulation experiments, we find that Thompson Sampling outperforms a number of alternative proposals, including methods specifically designed for settings with delayed feedback.
\end{abstract}

% Paper body
\section{Introduction}

The stochastic multi-armed bandit (MAB) problem is a framework for sequential experimentation that has been widely used in a number of application areas, including online advertising and recommendations \cite{lihong2011} and medical trials \cite{Gittins1979BanditPA, Press22387}. In the basic stochastic MAB specification, each of $k = 1, \, \ldots, \, K$ available arms has a reward distribution and, when an agent selects an arm, a reward drawn from the corresponding reward distribution is immediately revealed to them
\cite{Lai1985AsymptoticallyEA,Auer2004FinitetimeAO,Even-dar02pacbounds,lihong2011,pmlr-v23-agrawal12}.
In many real world settings, however, the assumption that rewards are revealed immediately following
an action is not applicable. For example, in a clinical trial it may take time to assess whether a patient
has responded to the given treatment \cite{FEROLLA20171652,bayesian_clinical,jiajing2014},
while in an e-mail marketing campaign, it may take time to see whether a user clicks on an ad
\cite{chapelle2014}.

Motivated by this observation, there has been considerable recent interest
on methods and theory for MAB problems with delays between when an action is taken, and when the
correponding reward is observed 
\cite{dudik2011,joulani2013,Mandel_Liu_Brunskill_Popovic_2015,vernade17,PikeBurke2018BanditsWD,Zhou2019LearningIG,Vernade2020LinearBW,gael20a,MAB_unrestricted_delay}.
These papers, however, all require either modifying familiar MAB algorithms to account for delays, or
propose new, delay-robust algorithms that are likely unfamiliar to practitioners.

The focus of this paper is in understanding how Thompson sampling \cite{Thompson1933ONTL},
used out of the box and without any adaptations, behaves under delays (as discussed further
below, we consider a specification where the posterior beliefs underlying Thompson sampling
are updated whenever new rewards are observed, and otherwise we proceed as usual). Thompson
sampling is a robust MAB algorithm that consistently achieves strong empirical performance
across a number of benchmarks, and is popular among practitioners \cite{lihong2011,russo2020tutorial}.

Our main result is that, using analytic ideas that build on results for both Thompson sampling
without delays \cite{Agrawal2013FurtherOR, Agrawal2013ThompsonSF,Kaufmann2012ThompsonSA} and
recent ideas for accommodating delays \cite{MAB_unrestricted_delay}, we can verify that
Thompson sampling admits strong formal guarantees in the setting with unrestricted delays.
Specifically we prove a $O(\sum_{i} {\log T}\,/\,{q_i} + d_i(q_i))$ regret bound for any selected
$q_i \in (0, \, 1)$, where $d_i(q_i)$ is the $q_i$-th quantile of the delay distribution of the $i$-th arm.
Meanwhile, our experiments align with observed strong empirical performance of Thompson
sampling: In a number of simulation specifications adapted from recent papers that
propose new delay-robust MAB algorithms, we find that Thompson sampling matches or
outperforms the proposed algorithms. Overall, our results suggest Thompson sampling
to be a robust and reliable method for stochastic MAB problems with delays.

% Head 2
\subsection{Related Work}

Early results on MAB with delays made strong assumptions on delay distributions:
For example, \citet{dudik2011} consider a model with constant (deterministic) delays,
\citet{Mandel_Liu_Brunskill_Popovic_2015} assume bounded delays,
while \citet{joulani2013} assume that the delay distribution has bounded expectation.
More recently, there has been interest in guarantees that are robust to heavy-tailedness
in the delay distribution. \citet{gael20a} consider a setting with polynomial tail bounds
on the delay distribution while in a recent advance, \citet{MAB_unrestricted_delay}
developed algorithms based on UCB and successive elimination that allow for unrestricted
delay distributions. In this paper, we also allow for unrestricted delay distributions,
and adapt ideas from \citet{MAB_unrestricted_delay} in order to do so.

There are a considerable number of results on the behavior of Thompson sampling
without delays \cite{pmlr-v23-agrawal12, Agrawal2013FurtherOR, Agrawal2013ThompsonSF, agarawalgoyal2017,Russo2016AnIA, russo2020tutorial, Kaufmann2012ThompsonSA,Gopalan2014ThompsonSF, Kawale2015EfficientTS, Ferreira2018OnlineNR}.
However, while Thompson sampling is known empirically to be robust to delays \cite{lihong2011},
we are not aware of formal regret guarantees available in this setting. \citet{joulani2013}
propose a meta-algorithm for turning any stochastic MAB algorithm with guarantees in the delay-free
setting into one that has guarantees with delays. However, when applied to Thompson sampling,
their meta-algorithm would require subtle modifications to Thompson sampling, and their results
only apply to delay distributions with bounded expectation.
\citet{gur2019adaptive} show that Thompson
sampling has desirable properties in a setting without delays, but where the analyst may sometimes
acquire additional information from external sources. At a high level, our paper is aligned with
\citet{gur2019adaptive} in that we both find Thompson sampling to be robust to non-standard information flows.
\citet{qin2022adaptivity} propose a robust variant of Thompson sampling that, when used for arm selection,
is guaranteed never to perform much worse than a uniformly randomized experiment in choosing a good
arm to deploy---even in non-stationary environments and under arbitrary delays.

Finally, we also note recent work on sequential learning with delays that go beyond the MAB-based
specification considered here.  \citet{vernade17} studied partially observed stochastic Bernoulli
bandit where only a reward of 1 can be directly observed with the value of delay. Under the assumption
that the delay distribution is known they provided algorithms that have close to optimal asymptotic
performance. \citet{Vernade2020LinearBW} further extended this framework to linear bandit where a version of Thompson Sampling was provided but no theoretical analysis was given.
\citet{PikeBurke2018BanditsWD} studied the MAB with delays when we only have access to aggregated anonymous feedback. \citet{Zhou2019LearningIG} studied a generalized linear contextual bandit with delays, while
\citet{Vernade2020NonStationaryDB} considered the case of non-stationary bandits with delays
when intermediate observations are available.  

\section{Problem Setup and Background} \label{sec:setup}
We consider the following problem setup, which adds the stochastic delay structure to the classical stochastic multi-armed bandit problem. Suppose we have $K$ arms with reward distribution $\nu_1,..., \nu_K$.
We assume that all reward distributions are supported on the interval $[0,1]$,
and that there is a unique optimal arm (with the highest mean reward).
At each round $t = 1,..., T$, the agent chooses an action $a_t$. The environment samples a reward $r_t$ from $\nu_{a_t}$, a delay $l_t$, and the agent observes the reward at round $t + l_t$. The setup is described in Model \ref{protocol:MAB_delay}. We assume for simplicity that delays are supported on $\mathbb{N}\bigcup\{\infty\}$; this is without loss of generality since the agent only collects feedback
and chooses new actions at integer time points.
We also note that the values of delay and the original time of the reward are not revealed to the agent. 

 \begin{algorithm}[ht]
    \SetAlgorithmName{Model}{}{}
        \For{$t \in [T]$}{
            Agent picks an action $a_t \in [K]$. \\
            Environment samples $r_t \sim \nu_{a_t}$ and $l_t$. \\
            Agent get a reward $r_t$, which is not immediately revealed. \\
            The set $\{(a_s, r_s): t = s + l_s\}$ is revealed to the agent, which contains values of action, reward pairs from previous rounds.
        }
     \caption{Stochastic multi-armed bandit with delays}
     \label{protocol:MAB_delay} 
\end{algorithm}

Our setup leaves the delay structure unspecified. In this paper, we mainly focus on one particular form of delay, i.i.d delays. In this setting, each arm has a separate delay distribution and $l_t$ is sampled independently of everything else from the delay distribution of arm $a_t$. We will evaluate our agent by the expected regret (called regret from hereon) which under Model \ref{protocol:MAB_delay} can be expressed as 
\begin{align}
        R_T &= T\mu_{i^*} - \sum_{t=1}^{T} \EE{r_t}
    = \sum_{i=1}^{K} \Delta_i\EE{m_T(i)}, \label{eq:regret}
\end{align}
where $\mu_i$ is the mean of distribution $\nu_i$ and $i^*$ denotes the optimal index. $\Delta_i = \mu_{i^*} - \mu_i$ is the gap between the optimal arm and the arm $i$ and $m_T(i)$ is the number of time we pull arm $i$ by time $T$. 

\subsection{Thompson Sampling under Delay}
Thompson sampling \cite{Thompson1933ONTL} is an adaptive Bayesian method that chooses the
actions at each round according to the current posterior probability that the action maximizes expected reward;
see \citet{russo2020tutorial} for a recent review. We consider a Bernoulli
bandit algorithm that starts with a uniform prior (Beta(1, 1)) on the mean parameter $\mu_i$ of each arm;
see Algorithm \ref{algo:TS_delay} for details. We note that our analysis below extends naturally to reward distributions with bounded support following the argument in \citet{pmlr-v23-agrawal12}. The main idea here is that the algorithm emerging from Beta-Bernoulli bandits can in fact be applied (and has good regret properties) for any setting with bounded outcomes. However, we do note that the analysis below does not accommodate reward distributions with unbounded support. 
 \begin{algorithm}[ht]
    \For{$i=1,...,K$}{
    Set counters $S_i = 0, F_i = 0$.
    }
    \For{$t \in [T]$}{
        For $i\in[K]$, sample $\theta_i(t)$ from Beta($S_i+1, F_i+1$). \\
        Play $a_t = \argmax_i \theta_i(t)$. \\
        \For{Revealed observation $(a_s, r_s)$ with $s + l_s = t$}{
            $S_{a_s} = S_{a_s} + r_s$ \\
            $F_{a_s} = F_{a_s} + 1- r_s$
        }
    }
    \caption{Thompson Sampling for Bernoulli Bandits under Delays}
    \label{algo:TS_delay} 
\end{algorithm}

\section{Theoretical Results} \label{sec:theory}
In this section we present formal regret guarantees of Algorithm \ref{algo:TS_delay} applying to Bernoulli bandit with i.i.d delays. 
We make the following assumptions on top of Model \ref{protocol:MAB_delay}.
\begin{itemize}
    \item \textbf{Assumption 1:} The distributions $\nu_1,...,\nu_K$ are Bernoulli distributions with means $\mu_1 > \cdots > \mu_K$.
    \item \textbf{Assumption 2:} Each arm has a delay distribution $\mathcal{D}_1,...,\mathcal{D}_K$ supported on non-negative integers and $\infty $ and $l_t$ is sampled from $\mathcal{D}_{a_t}$, which is independent of rewards and past delays.
\end{itemize}
We use the following notation throughout:
\begin{align*}
&\text{\#actions taken: } m_t(i) = \sum_{s = 1}^t 1(\{a_s = i\}), \\
&\text{\#available observations: } n_t(i) = \!\!\! \sum_{\{s : s + l_s \leq t\}} \!\!\! 1(\{a_s = i\}), \\
&\text{delay quantile: } d_i(q) = \inf\{d : \mathbb{P}[l_t \leq d \cond A = i] \geq q\}.
\end{align*}
We also write $\theta_i(t)$ for the posterior sample of arm $i$ at time $t$.

\subsection{Two-arm case}

For simplicity, we start by considering the case with $K = 2$, as this enables us to present a proof with
less notational overhead. Recall that, by assumption, $\mu_1 > \mu_2$; and we write the arm gap as $\Delta = \mu_1-\mu_2$.
In this setting, we show the following.

\begin{theo}\label{thm:2-arm}
Under assumption 1 and 2 with $K=2$, suppose further we have i.i.d delays for each arm. Then the regret
\eqref{eq:regret} of Algorithm~\ref{algo:TS_delay} is bounded by 
\begin{align*}
    \min_{q_1,q_2 \in (0,1]}& \frac{48\log T}{q_2\Delta} + \frac{6}{\Delta}\left(\frac{32\log T}{q_1\Delta} + d_1(q_1)\Delta+\Delta \right)
    \\&  + d_2(q_2)\Delta + O\left(\frac{1}{\Delta} + \frac{1}{\Delta^3}\right)
\end{align*}
\end{theo}

In other words, the regret of a Bernoulli bandit with delays and $K = 2$ arms is bounded to order
\begin{equation}
\label{eq:K2simple}
R_T = O\bigg(\frac{\log T}{\Delta}\left(\frac{1}{q_1}+\frac{1}{q_2}\right)  +(d_1(q_1)+d_2(q_2))\Delta\bigg)
\end{equation}
for any choice of $q_1$ and $q_2$. For example, if we set $q_1 = q_2 = 0.5$, then
the above bound depends on the medians of the delay distributions.

As discussed above, most earlier results on MAB with delays---using Thompson sampling or
other algorithms---made further assumptions on the delay distributions (e.g., \citet{joulani2013} assumed
bounded expectations for delays) and so our results are not directly comparable to them.
Only recently, \citet{MAB_unrestricted_delay} obtained bounds that hold with unrestricted delays
for some variants of UCB and successive elimination. With $K = 2$, their bound is
\begin{equation}
\label{eq:L2}
R_T \leq \min_{q_1, q_2} \frac{40\log T}{\Delta}\left( \frac{1}{q_1} + \frac{1}{q_2}\right)  + \log(2)(d_1(q_1) + d_2(q_2))\Delta
\end{equation}
for a variant of the Successive Elimination algorithm of \citet{Even-dar02pacbounds} adapted to the setting with delays. 
We see both bounds have an extra term involving linear combination of the quantiles of the delay distribution and both are of the same big-$O$ order. In particular in the case of a constant delay of value $d$, both bounds have an additive $O(d)$ term compared to the regret bound without delay.

We note that the similar appearance of the bounds \eqref{eq:K2simple}
and \eqref{eq:L2} is not a coincidence: As seen below, our proof involves incorporating
some key ideas from \citet{MAB_unrestricted_delay} into a study of Thompson sampling that builds
on \citet{pmlr-v23-agrawal12}.

\begin{proof}[Proof of Theorem~\ref{thm:2-arm}]
When $K = 2$, regret measures how often we pull the second
arm in expectation,
\begin{equation}
\label{eq:startK2}
R_T = \Delta \mathbb{E}[m_T(2)],
\end{equation}
and so to bound regret at $T$ we need to bound $m_T(2)$. To this end,
we decompose the problem as follows:
\begin{itemize}
\item Let $Y_j$ be the number of times the 2nd arm is pulled between $j$-th
and $(j+1)$-st draws of the first arm, $Y_j = |\{t : m_t(1) = j\}| - 1$.
\item Let $\tau_2 = \inf\{t : n_t(2) \geq 24 \Delta^{-2} \log(T)\}$.
\item Let $j_0 = m_{\tau_2}(1)$. 
\end{itemize}
Because the agent must pull the 2nd arm each time they don't pull the 1st, we see that
\begin{equation}
\label{eq:decomp}
m_T(2) \leq m_{\tau_2}(2) + \sum_{j = j_0}^{m_1(T)} Y_j.
\end{equation}
To proceed, we then posit the following events,
\begin{equation}
\begin{split}
&F_1 = \biggl\{ \exists t\le T,i: m_t(i) \geq \frac{24\log(T)}{q_i}, n_{t+d_i(q_i)}(i) < \frac{q_i}{2}m_t(i)\biggr\}, \\
&F_2 = \left\{\exists t \le T: \theta_2(t) > \mu_2 + \frac{\Delta}{2}, n_2(t) \geq \frac{24\log T}{\Delta^2}\right\}, \\
\end{split}
\end{equation}
and argue that these two events are rare. Our definition of the events $F_1$
and $F_2$ is motivated by ideas used in both \citet{pmlr-v23-agrawal12} and
\citet{MAB_unrestricted_delay}.

\begin{lemm} \label{lemm:concentration} Under the conditions of Theorem \ref{thm:2-arm},
$\mathbb{P}[F_1] \leq \frac{1}{T}$ and $\mathbb{P}[F_2] \leq \frac{2}{T}$.
\end{lemm}
\begin{proof}
See section \ref{proof:lemm3.2} of the appendix. 
\end{proof}

The upshot is that we can now define a ``good'' event $G = \neg F_1 \bigcap \neg F_2$,
and note that by the union bound, $\mathbb{P}[G] \geq 1 - 3 /T$. Thus,
\begin{equation}
\label{eq:focus_good}
\begin{split}
\EE{m_T(2)} &\leq \mathbb{P}[G] \mathbb{E}[m_T(2)  \, |\, G]  + \mathbb{P}[\neg G] \mathbb{E}[m_T(2)  \, |\, \neg G] \\
&\leq \mathbb{E}[m_T(2) \, |\, G] + 3.
\end{split}
\end{equation}
We now plug \eqref{eq:decomp} into \eqref{eq:focus_good}. Our goal at this point
is to prove that
\begin{align}
&\mathbb{E}[m_{\tau_2}(2)  \, |\, G] \leq \frac{48\log T}{q_2\Delta^2} + d_2(q_2), \text{ and} \label{eq:goal1} \\
&\mathbb{E}\left[\sum_{j = j_0}^{m_1(T)} Y_j  \, \Big|\, G\right] \leq O\left(\frac{1}{\Delta^2} + \frac{1}{\Delta^4}\right) + \left(\frac{32\log T}{q_1\Delta^2} + d_1(q_1)+1 \right) \frac{6}{\Delta}  \label{eq:goal2}
\end{align}
Combining \eqref{eq:startK2}, \eqref{eq:decomp}, \eqref{eq:focus_good}, and the equations
above then yields the desired result.

Now, to check \eqref{eq:goal1}, we note that on $G$,
if $m_{\tau_2 - d_2(q_2)}(2) \ge \frac{24\log(T)}{q_2}$ then
\[
n_{\tau_2}(2) = n_{\tau_2 - d_2(q_2) + d_2(q_2)}(2)  \geq \frac{q_2}{2}m_{\tau_2 - d_2(q_2)}(2).
\]
This implies 
\begin{align*}
m_{\tau_2 - d_2(q_2)}(2) & \leq \max\left\{\frac{48\log T}{q_2\Delta^2}, \frac{24\log(T)}{q_2}\right\}\\
& = \frac{48\log T}{q_2\Delta^2} \,\, \text{by our assumption } \Delta \in (0,1] 
\end{align*}
Hence, we have 
\begin{align*}
    m_{\tau_2}(2) & = m_{\tau_2 - d_2(q_2)}(2) + m_{\tau_2}(2) - m_{\tau_2 - d_2(q_2)}(2) \\
    & \leq \frac{48\log T}{q_2\Delta^2} + d_2(q_2),
\end{align*}
and so in fact the inequality \eqref{eq:goal1} holds almost surely conditionally on $G$.

Next, to check \eqref{eq:goal2}, we bound $\EE{\sum_{j=1}^{T}Y_j\cond G}$ instead. Let $t_j = \inf\{t:m_t(1) = j\}$ be the time we pull arm 1 the $j$-th time. Let $X(j,s,y)$ denote the number of trials before a Beta$(s+1,j-s+1)$ exceeds $y$ as in \citet{pmlr-v23-agrawal12}. We proceed by proving the following lemma.
\begin{lemm} \label{lemm:bound_Yj}
Let $s(j)$ be the number of successes among the $j$ observed rewards from arm 1. Then for $j \ge j_0$, 
\begin{align}
&\sum_{j=1}^{T}\EE{Y_j\cond G} \notag \\
&\le \sum_{j=1}^{T}\EE{\min\{X(n_{t_j}(1), s(n_{t_j}(1)), \mu_2+\frac{\Delta}{2})), T\}} \notag  \\
& + \sum_{k=1}^{T}\EE{\min\{X(k, s(k), \mu_2+\frac{\Delta}{2}), T\}}\label{eq:bound_Yj}
\end{align}
\end{lemm}
\begin{proof}
See section \ref{proof:lemm3.3} of the appendix.
\end{proof}
To bound right side of \eqref{eq:bound_Yj}, we use the following lemma from \cite{agarawalgoyal2017} and its corollary.

\begin{lemm}
\label{lem:exj}
For any $i \ne 1$, let $y_i = \mu_i + \frac{\Delta_i}{2}$. Let $D_i$ denote the KL-divergence between $\mu_1$ and $y_i$, i.e.
 $D_i= y_i \log \frac{y_i}{\mu_1} + (1-y_i) \log \frac{1-y_i}{1-\mu_1}.$ Then
%Note that $D\ge 0$. 
\begin{eqnarray*}
 \EE{ \min\{X(k, s(k), y_i), T \} }& \le & \ \left\{\begin{array}{ll}
   \displaystyle  \frac{6}{\Delta_i},&  k < \frac{6}{\Delta_i}\\
   \displaystyle O\left(e^{-\Delta_i^2k/8}+\frac{4}{(k+1)\Delta_i^2}e^{-D_i k}+ \frac{1}{e^{\Delta_i^2 k/16}-1}\right) , & k \ge \frac{6}{\Delta_i}\\
   \end{array}\right.
\end{eqnarray*}
In particular, when $k \ge \frac{16\log T}{\Delta_i^2}$, $\EE{ \min\{X(k, s(k), y_i), T \} } = O(\frac{1}{T})$
\end{lemm} 
\begin{coro} \label{coro:sum_Xk}
Under the assumption of Lemma \ref{lem:exj}, we further have the following bound 
\begin{equation*}
    \sum_{k=1}^{T} \EE{ \min\{X(k, s(k), y_i), T \} }
    \le  O\left(\frac{1}{\Delta_i^2} + \frac{1}{\Delta_i^2D_i} + \frac{1}{\Delta_i^4}\right) 
\end{equation*}
\end{coro}
Now we bound $n_{t_j}(1)$. Condition on $G$, since $G \subset \neg F_1$, we know if $m_t(1) \geq \frac{32\log T}{q_1\Delta^2}$ then $n_{t+d_1(q_1)}(1) \ge \frac{q_1}{2}m_t(1) \ge \frac{16 \log T}{\Delta^2}$. Hence,  let $l = \frac{32\log T}{q_1\Delta^2}$, we have $m_{t_l}(1) \ge \frac{32 \log T}{q_1\Delta^2}$, and 
\[
n_{t_j}(1) \ge n_{t_l+d_1(q_1)}(1) \ge \frac{16 \log T}{\Delta^2}.
\]
Let $M = \lceil \frac{32\log T}{q_1\Delta^2} + d_1(q_1)\rceil$, $y = \mu_2 + \frac{\Delta}{2}$ and $D = y \log \frac{y}{\mu_1} + (1-y) \log \frac{1-y}{1-\mu_1}$ we then have 
\begin{align*}
    &\EE{\sum_{j=1}^{T}Y_j\cond G}\\
    &\le \sum_{j = 1}^{T}\EE{\min\{X(n_{t_j}(1), s(n_{t_j}(1)), \mu_2+\frac{\Delta}{2})), T\}\cond G} + \\
    &+ \sum_{k=1}^{T}\EE{\min\{X(k, s(k), \mu_2+\frac{\Delta}{2}), T\}\cond G} \\
    & \le \sum_{j = 1}^{M}\EE{\min\{X(n_{t_j}(1), s(n_{t_j}(1)), \mu_2+\frac{\Delta}{2})), T\}\cond G} \\
    &+ \sum_{j = M+1}^{T}\EE{\min\{X(n_{t_j}(1), s(n_{t_j}(1)), \mu_2+\frac{\Delta}{2})), T\}\cond G} \\
    &+O\left(\frac{1}{\Delta^2} + \frac{1}{\Delta^2D} + \frac{1}{\Delta^4}\right) \\
    &\le M\max_n \EE{\min\{X(n, s(n), \mu_2+\frac{\Delta}{2})), T\}\cond G} + \frac{16}{T}T \\
    &+O\left(\frac{1}{\Delta^2} + \frac{1}{\Delta^2D} + \frac{1}{\Delta^4}\right) \\
    & \le \frac{6M}{\Delta} + O\left(\frac{1}{\Delta^2} + \frac{1}{\Delta^2 D} + \frac{1}{\Delta^4}\right) 
\end{align*}
\end{proof}

\subsection{Multi-arm case}
Now we present the results in the case of more than 2 arms. In this case we have the following result. 
\begin{theo}\label{thm:n-arm}
Under assumption 1 and 2 with $K>2$, suppose further we have i.i.d delays for each arm. Then the regret
\eqref{eq:regret} of Algorithm \ref{algo:TS_delay} is bounded by 
\begin{align*}
    & \min_{q_i \in (0,1]} \sum_{i=2}^{K} \frac{48\log T}{q_i\Delta_i} + d_i(q_i)\Delta_i \\
    & + \sum_{i=2}^{K} O\left(\frac{1}{\Delta_i} + \frac{1}{\Delta_i^3}\right) + \left(\frac{32\log T}{q_1\Delta_i} + d_1(q_1)\Delta_i+\Delta_i \right)\frac{6}{\Delta_i} \\
    & + \sum_{i=2}^{K} O\left(\frac{1}{\Delta_i} + \frac{1}{\Delta_i^3}\right) \left(\sum_{i=2}^{K}\frac{48\log T}{q_i\Delta_i^2} + d_i(q_i)\right)+4(K-1) 
\end{align*}
\end{theo}
\begin{proof}
See section \ref{proof:thm3.6} of the appendix.
\end{proof}
\begin{coro}
The regret of Algorithm \ref{algo:TS_delay} when $K > 2$ is of order 
\[
O\left(\sum_{i\ne1} \frac{\log T}{\Delta_i}\left(\frac{1}{q_1}+\frac{1}{q_i}\right)+(d_1(q_1)+d_i(q_i))\Delta_i\right)
\]
\end{coro}
\textbf{Remark:} The bound in \citet{MAB_unrestricted_delay} in this case is 
\begin{align}
\label{eq:logK}
    & \min_{q_i} \sum_{i\ne 1} \frac{40 \log T}{\Delta_i}\left(\frac{1}{q_1}+\frac{1}{q_i}\right)
     + \log(K)\max_{i\ne 1}\{(d_1(q_1)+d_i(q_i))\Delta_i\}.
\end{align}
If we consider the number of arms $K$ to be fixed then this bound has the same scaling as ours.
However, if $K$ can grow, then the cost of delays in the bound of \citet{MAB_unrestricted_delay} has a better
$K$-scaling than ours (logarithmic as opposed to linear). It would be an interesting topic for further work
to see if an alternative proof could establish $\log(K)$-scaling for the cost of delays in Thompson sampling,
or if there exists a separation in results one can get from Thompson sampling versus the algorithms considered in \citet{MAB_unrestricted_delay}.
In any case, we note that in our experiments we consider many settings with a moderately large number of arms,  $K=20$,
and the good behavior of Thompson sampling there suggests that it is at least moderately robust to the large-$K$ setting.

\section{Numerical Experiments} \label{sec:experiments}

We have shown that in stochastic multi-armed bandit with i.i.d.~delays, Thompson Sampling can achieve comparable regret bounds as a variant based on Successive Elimination, which constructs upper confidence bounds. However, the advantage of Thompson Sampling under delays goes far beyond achieving good theoretical bounds. In this section, we demonstrate through extensive experiments that Thompson Sampling can often outperform a number of UCB variants under various delay structures. Specifically, we will compare Thompson Sampling with baselines under a number of delay settings considered in prior work, as well as in some new settings (including ones with non-i.i.d.~delays). Note that our implementation of Thompson Sampling does not change depending on the assumptions of delays, but the UCB variants used as baselines change from setting to setting in order to accommodate different delay distributions. We use the Bernoulli bandit setting for all the experiments.    

Table \ref{tab:experiment_summary} gives a summary of all settings we consider. We explicitly list the type of delays as well as whether the setup is considered in previous work. 
% Note use of \abovespace and \belowspace to get reasonable spacing
% above and below tabular lines.

\begin{table}[htb]
\centering
\begin{tabular}{|ccc|}
 \hline
i.i.d? & Delay Type & Reference\\
\hline
Yes    & Fixed & \citet{MAB_unrestricted_delay} \\
Yes & $\alpha$-Pareto & \citet{gael20a}\\
Yes    & Packet-loss & \citet{MAB_unrestricted_delay}\\
Yes    & Geometric & \citet{vernade17}\\
Yes     & Uniform &  $\times$ \\
No      & Queue-based &   $\times$ \\
   \hline
\end{tabular}
\caption{Summary of all experimental settings we consider}
\label{tab:experiment_summary}
\end{table}

\subsection{Methods}
We consider the following methods for all experiments and also include methods designed to target certain settings in previous work if there is any.
\begin{itemize}
    \item Delayed-UCB1, straightforward adaptation of UCB1 \cite{Auer2004FinitetimeAO} in delay case outlined in \citet{joulani2013}. We employ random tie breaking.
    \item Successive Elimination with delays (SE), which is shown to achieve great theoretical and empirical performance with unrestricted delays in \citet{MAB_unrestricted_delay}. We show the algorithm details from their work in Algorithm \ref{algo:SE_delay}. We use the same radius when constructing upper confidence bounds as \cite{MAB_unrestricted_delay}.
    \item Thompson Sampling with delays (TS), detailed in Algorithm \ref{algo:TS_delay}.
\end{itemize}
Furthermore, whenever we reuse simulations from prior work (see Table \ref{tab:experiment_summary}),
we also consider as baselines algorithms proposed in the corresponding papers. Specifically,
for $\alpha$-Pareto delays we also consider the PatientBandit (PB) algorithm of \citet{gael20a},
for the packet-loss setting we also consider the Phased Successive Elimination (PSE) algorithm of 
\citet{MAB_unrestricted_delay}, and for geometric delays we also consider the UD-UCB algorithm
of \citet{vernade17}.

The main algorithms we consider, namely Delayed-UCB1, SE and TS, are all agnostic
to the delay distribution, i.e., the algorithm itself doesn't explicitly depend on what
we assume about delays. In contrast, the other baselines we use may depend on the delay
distributions---and when this is the case we let the algorithm use oracle information about
the delays (thus making our comparison potentially too favorable towards these baselines).

 \begin{algorithm}[ht]
   \SetKwInOut{Initializations}{Initializations}
   \KwIn{Number of rounds $T$, number of arms $K$}
   \Initializations{$S \gets [K]$, $t \gets 1$}
    \While{$t < T$}{
        Pull each arm $i \in S$ \\
        Observe any incoming feedback \\
        Set $t \leftarrow t + |S|$ \\
        Update lower and upper confidence bounds where the radius is $\sqrt{\frac{2}{\max\{n_t(i),1\}}}$ \\
        Remove from $S$ all arms $i$ such that exists $j$ with $UCB_t({i})<LCB_t({j})$
    }
    \caption{Successive Elimination with Delays}
    \label{algo:SE_delay}
\end{algorithm}

\subsection{I.I.D Delays}
In this part we mainly focus on experiments with i.i.d delays considered in Section \ref{sec:setup}.

\textbf{Fixed Delays.} In this setting, all the delays have a fixed value. As in \citet{MAB_unrestricted_delay}, we set this value to be 250. We fix the number of arms to be $K = 20$ and the means of each arm are generated uniformly from the interval $[0.25, 0.75]$. The maximum round is $T=20000$. All results are averaged cross 100 replications. Figure \ref{fig:fixed_delay_regret} shows the resulting regret plot\iffalse when we only change to random tie breaking as this is more common in the field when using UCB style algorithm\fi. We see that TS significantly outperforms the UCB style algorithms and the cumulative regret plot plateaus long before the other two methods.

\begin{figure}[htb]
\vskip 0.1in
\begin{center}
\centerline{\includegraphics[scale=0.25]{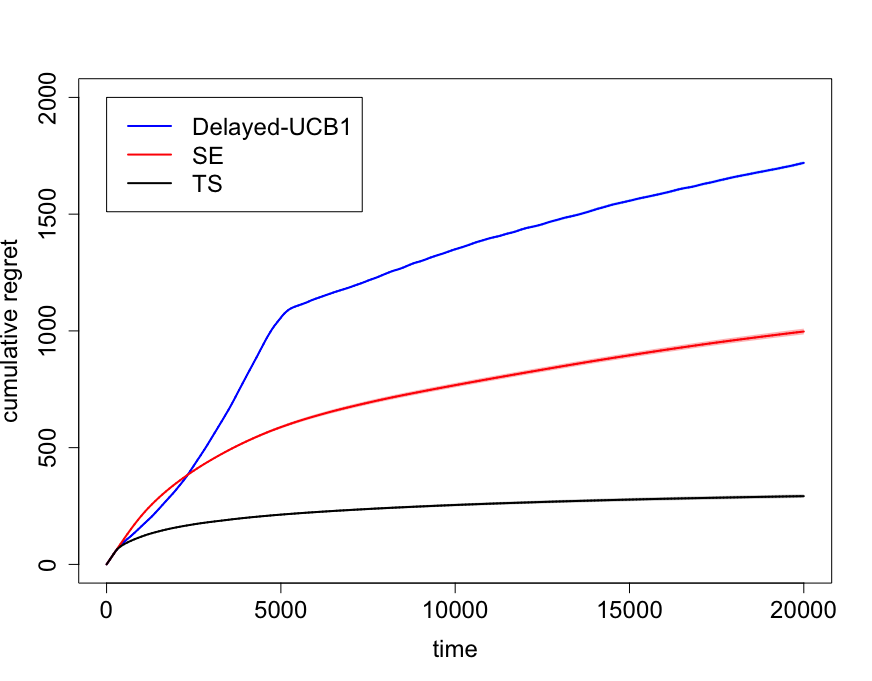}}
\caption{Regret of Delayed-UCB1, SE, TS for fixed delay 250, averaged over 100 replications. Error bars are displayed via shading.}
\label{fig:fixed_delay_regret}
\end{center}
\vskip -0.1in
\end{figure}
We also experimented with increasing and decreasing delays. As expected, the gap between UCB style algorithms and TS is not significant when the delay is 0 and widens as we increase the delay. 

\textbf{$\alpha$-Pareto delays.} These delays are distributions with heavy tails and infinite expectations with $\alpha \le 1$. The parameter $\alpha$ controls the tail behavior with heavier tails for smaller $\alpha'$s. We employ the experimental setup in \citet{gael20a} which has $K=2$ arms and $T=3000$ rounds. We fix the mean parameters to be $\mu_1 = 0.4$ and $\mu_2 = 0.45$ to make the problem instance reasonably difficult. We further let $\alpha_1 = 1$, which controls the tail of the delay distribution of the first arm. We vary the value of $\alpha_2$ and let it change in the set $\{0.2, 0.5, 0.8\}$. 

In case of $\alpha$-Pareto Delays, \citet{gael20a} proposed the PatientBandit (PB) algorithm which requires an input of oracle $\alpha$ that captures the tail decay of the delay distributions to construct upper confidence bounds. PB is a robust algorithm that, as shown in \citet{gael20a}, can also handle a partially observed setting where a feedback of 0 can mean either a reward of 0 or the delay has not passed. Figure \ref{fig:pareto_delay_regret} shows the resulting cumulative regret of all four algorithms averaged over 300 replications. 

\begin{figure}[htb]
\vskip 0.1in
\begin{center}
\centerline{\includegraphics[scale=0.25]{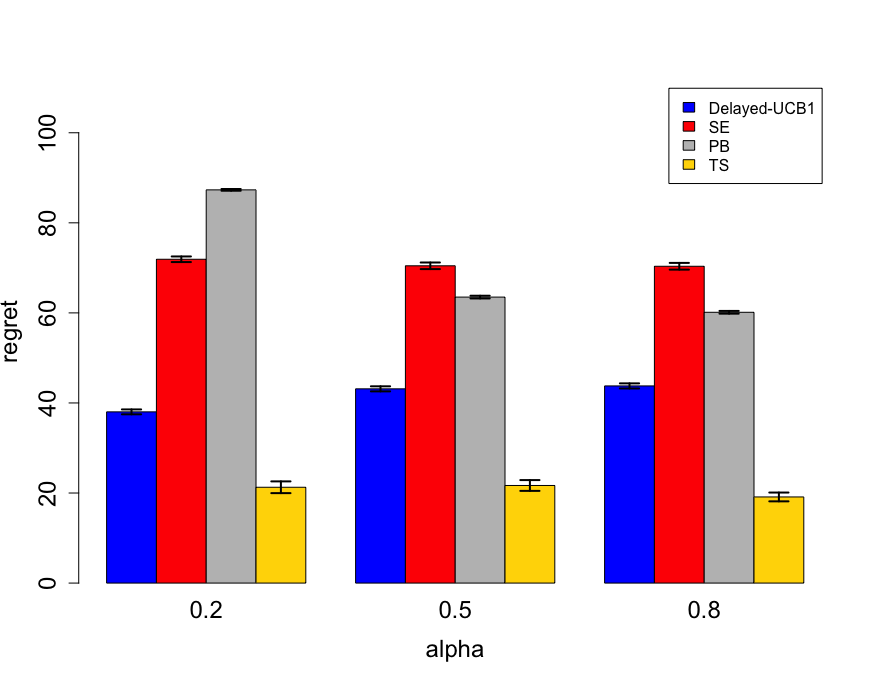}}
\caption{Regret of UCB without $\log$ factor, SE, TS and PB for Pareto distributed delays with varying $\alpha$ for the optimal arm 2, averaged over 300 replications}
\label{fig:pareto_delay_regret}
\end{center}
\vskip -0.1in
\end{figure}

We see that SE only outperforms PB when $\alpha$ is small, which means the delay distribution, i.e. when the delay distribution has a very heavy tail. This is partly due to the fact that PB algorithm uses a conservative upper confidence bound since it assumes a uniform $\alpha$ controlling the tails of all delay distributions. We also see that TS and Delayed-UCB1 are significantly better than SE and PB. In addition, TS has much smaller regret compared to Delayed-UCB1.

\textbf{Packet-loss.} This setting refers to a delay with infinite expectation. Specifically the delay will have a value of 0 with probability $p$ and infinity otherwise. In \citet{MAB_unrestricted_delay}, the Phased Successive Elimination (PSE) algorithm was provided specifically for this setting, which balances the observed reward counts of each arm during each phase and then does elimination based on upper and lower confidence bound. We run our experiments with $K = 20$ arms and the means of each arm are sampled uniformly from interval $[0.25, 0.75]$. We sample the probabilities $p$ of the packet loss uniformly from interval $[0,1]$. We average our results across 200 replications and run for $T=10000$ rounds in each replication. We use the same PSE algorithm details as in \citet{MAB_unrestricted_delay}.

\begin{figure}[htb]
\vskip 0.1in
\begin{center}
\centerline{\includegraphics[scale=0.25]{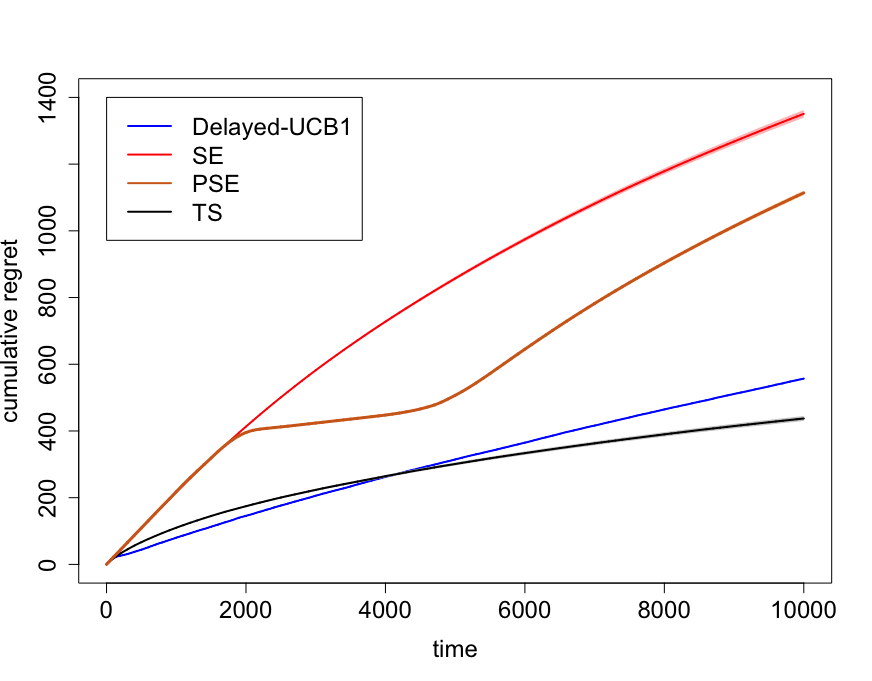}}
\caption{Regret of Delayed-UCB1, SE, PSE, TS algorithms under the packet loss setting where rewards are observed immediately with certain probability and not observed at all otherwise. Results are averaged over 200 runs. Error bars are displayed via shading.}
\label{fig:packet_regret}
\end{center}
\vskip -0.1in
\end{figure}

Figure \ref{fig:packet_regret} shows the resulting regret plot. We see that TS and Delayed-UCB stand out as the clear winners and TS performs slightly better. The reason that Delayed-UCB is comparable to TS is that in the parameter setting we use to run the experiments, the optimal arm has a small probability $(p=0.108)$ in packet loss. This means the probability of having infinite delay is close to 0.9 for the optimal arm. So the information about the optimal arm is very limited. In fact as we increase this probability we see a widening gap between the two. As we do not know in practice the optimal arms, we think even in this setting TS stands out as a safe choice to use. 

\textbf{Geometric Delays.} In this setting, the delays are Geometrically distributed, which means they can still be arbitrarily long but the expectation is finite. \citet{vernade17} proposed an algorithm which assumes a known delay distribution and used geometric delays in the simulation. Specifically, the algorithm uses CDF functions of the delay distribution to form a conditionally unbiased estimator. Then an upper confidence bound is formed to select which arm to pull at each round. We include this algorithm, which is called UD-UCB in \citet{vernade17} into our comparison. As in \citet{vernade17}, we ran for $T =10000$ rounds and let the means of three arms be $(0.5, 0.4, 0.3)$. We use $p= 0.01$ to sample the delay and average across 200 replications. 

Figure \ref{fig:geometric_delay_regret} shows the resulting regret plot. We see that TS performs much better than the other three methods and interestingly is also much better than UD-UCB even if UD-UCB knows the delay distribution and uses this extra information in the algorithm. This again shows the robustness of Thompson Sampling to various delays. 

\begin{figure}[htb]
\vskip 0.1in
\begin{center}
\centerline{\includegraphics[scale=0.25]{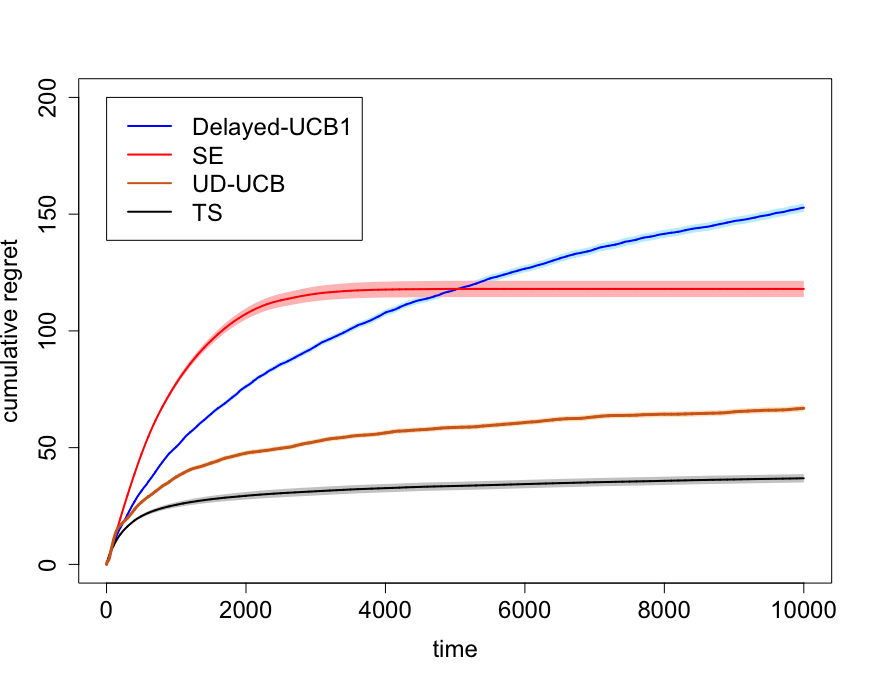}}
\caption{Regret of UCB, SE, UD-UCB, TS algorithms under delays sampled from Geometric$(0.01)$. Results are averaged over 200 runs. Error bars are displayed via shading.}
\label{fig:geometric_delay_regret}
\end{center}
\vskip -0.1in
\end{figure}

\textbf{Uniformly-distributed Delays.} Finally we consider a uniform delay distribution. We will use the same setup as experiments for fixed delays, namely $K=20$ arms and uniformly sampled means. However, instead of having a fixed value 250 we will sample delay uniformly from the integers in interval $[150, 300]$. In this setting we only consider the original three methods as we did not find any algorithm in the literature specifically designed to target such delay. We average across 100 replications. Figure \ref{fig:uniform_delay_regret} shows the resulting regret plot. We see again that TS performs significantly better than the other methods.

\begin{figure}[htb]
\vskip 0.1in
\begin{center}
\centerline{\includegraphics[scale=0.25]{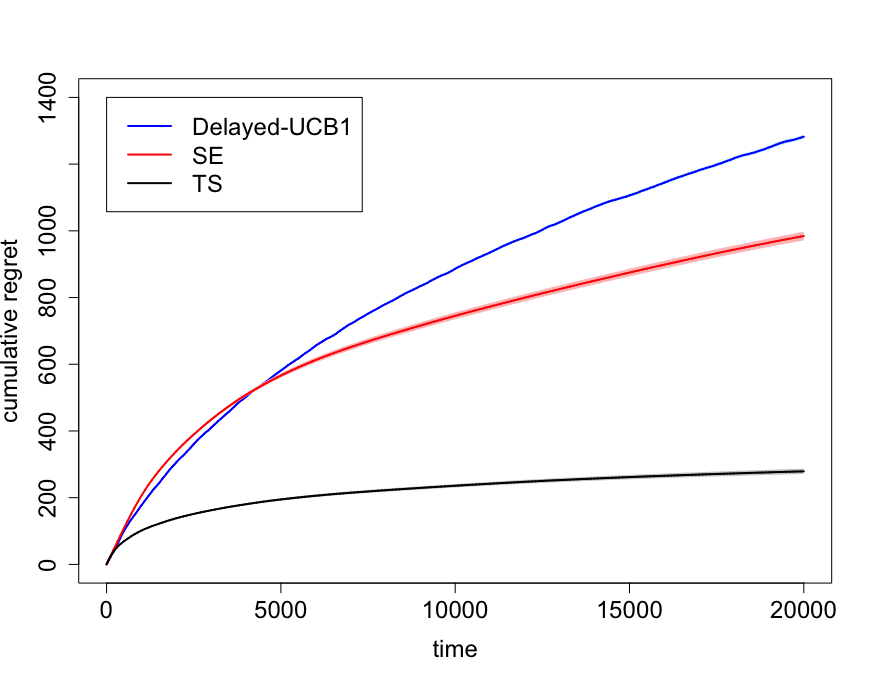}}
\caption{Regret of UCB, SE, TS algorithms under delays sampled uniformly from integers in $[0,300]$. Results are averaged over 100 runs. Error bars are displayed via shading.}
\label{fig:uniform_delay_regret}
\end{center}
\vskip -0.1in
\end{figure}

\subsection{Non i.i.d Delays} \label{sec: non-iid}
In our formal analysis we only considered the behavior of Thompson sampling under i.i.d delays.
Here, however, we seek to empirically validate its behavior with non-i.i.d. delays and find that,
although corresponding theory still remains to be developed, Thompson Sampling still appears to
achieve reasonable performance when facing such delays.

\textbf{Queue-based Delay.} We describe a queue-based delay mechanism here which is inspired by classical models for queuing system \cite{Kelly2011book}. The basic idea is once we select an arm $i$, the current action goes to the queue for arm $i$. If there is no other actions in the queue, i.e. the queue is empty then the reward is revealed immediately. Otherwise, the reward will be revealed when the other actions in the queue are cleared. The time to clear an action for each arm will be an exponential distribution with rate 0.1. Clearly from our description, this is a non i.i.d scenario. We let $K = 5$, sample means uniformly from $[0.25, 0.75]$ and average over 200 runs. Figure \ref{fig:queue_delay_regret} shows the resulting regret plot. We see that even under this non i.i.d delay scenario, Thompson Sampling works much better than the other two methods. 

\begin{figure}[htb]
\begin{center}
\centerline{\includegraphics[scale=0.25]{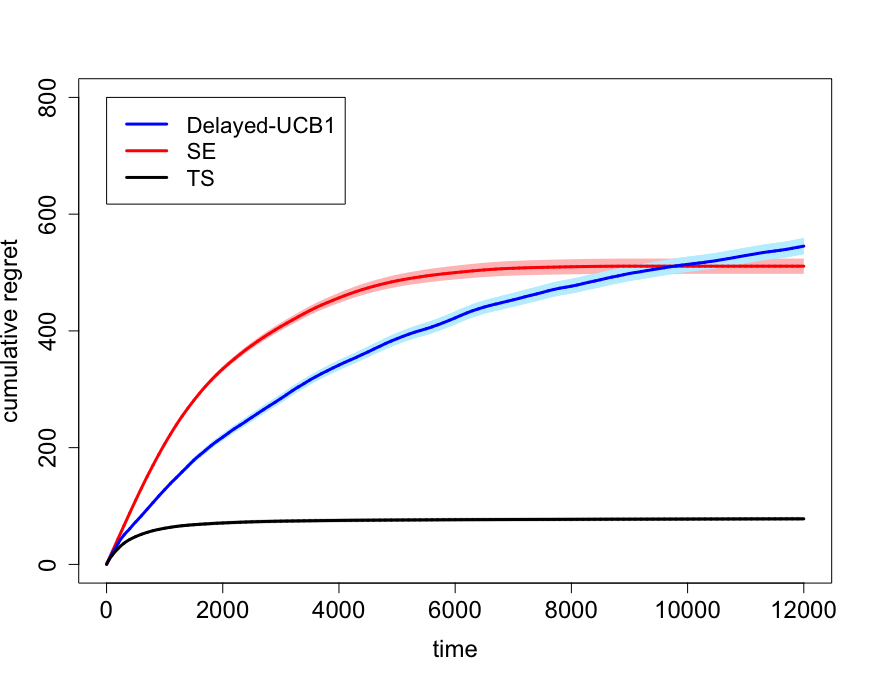}}
\caption{Regret of UCB, SE, TS algorithms under queue-based delays. Results are averaged over 200 runs. Error bars are displayed via shading.}
\label{fig:queue_delay_regret}
\end{center}
\end{figure}

\section{Discussion}
In this paper we presented a regret analysis of Thompson Sampling under i.i.d unrestricted delays which could have potentially unbounded expectations. The regret bound is of the same order as the UCB-based algorithms proposed in the literature. In addition to theoretical findings, we present a comprehensive empirical study of existing methods under various delay distributions, including unbounded ones with infinite expectations, and further consider a non i.i.d delay mechanism based on queues. In all cases, we find Thompson sampling to be a robust and performant algorithm that does not require any problem-dependent tuning. 

In future work, it would be of considerable interest to study the behavior of Thompson Sampling beyond the reward-independent i.i.d setting.
One example is to analyze queue-based delays we introduced in Section \ref{sec: non-iid}, where we have already empirically established that
Thompson Sampling is quite robust. Another example is to consider a reward-dependent setting
where the delay and the reward are sampled jointly from a distribution that allows for dependence (e.g., where small rewards come with longer delays).
In our experiments, we have found some reward-dependent settings where applying Algorithm \ref{algo:TS_delay} directly results in a large regret;
and it would be interesting to investigate variants of Thompson Sampling that may achieve better performance here.

\bibliographystyle{plainnat}
\bibliography{references}

% Appendix
\newpage
\appendix
\section{Proof of Lemma \ref{lemm:concentration}} \label{proof:lemm3.2}
\begin{proof}
We use the following lemma in \citet{MAB_unrestricted_delay}.
\begin{lemm}
At time $t$, for any arm $i$ and quantile $q \in (0,1]$, it holds that 
\[
\mathbb{P}\{n_{t+d_i(q)} < \frac{q}{2}m_t(i)\} \le \exp\left(-\frac{q}{8}m_t(i)\right)
\]
\end{lemm}
We have 
\begin{align*}
    \mathbb{P}(F_1) &= \mathbb{P}\biggl\{ \exists t\le T,i: m_t(i) \geq \frac{24\log(T)}{q_i}, n_{t+d_i(q_i)}(i) < \frac{q_i}{2}m_t(i)\biggr\} \\
    &\le \sum_{i}\sum_{t:m_t(i) \geq \frac{24\log(T)}{q_i}} \mathbb{P}\left\{ n_{t+d_i(q_i)}(i) < \frac{q_i}{2}m_t(i) \right\} \\
    & \le \sum_{i}\sum_{t:m_t(i) \geq \frac{24\log(T)}{q_i}} \exp\left(-\frac{q_i}{8}m_t(i)\right) \\
    & \le TK \frac{1}{T^3}\\
    &\le \frac{1}{T}
\end{align*}
For the second part we have from Lemma 6 in \cite{pmlr-v23-agrawal12} that 
\[
\forall t, \mathbb{P}\left\{ \theta_2(t) > \mu_2 + \frac{\Delta}{2}, n_2(t) \ge \frac{24 \log T}{\Delta^2}\right\} \le \frac{2}{T^2},
\]
taking a union bound over $t$ gives us the result. 
\end{proof}

\section{Proof of Lemma \ref{lemm:bound_Yj}} \label{proof:lemm3.3}
\begin{proof}
\begin{align*}
    & \sum_{j=1}^{T} \EE{Y_j\cond G} \\
    & \le \sum_{j=1}^{T} \EE{\max\{X(n_{t_j}(1), s(n_{t_j}(1)), \mu_2 + \frac{\Delta}{2}), ...,  X(n_{t_{j+1}}(1), s(n_{t_{j+1}}(1)), \mu_2 + \frac{\Delta}{2})\}\cond G} \\
    & \le \sum_{j=1}^{T} \EE{\sum_{k=n_{t_j}}^{n_{t_{j+1}}} X(k, s(k), \mu_2 + \frac{\Delta}{2})} \\
    &\le \sum_{j=1}^{T}\EE{\min\{X(n_{t_j}(1), s(n_{t_j}(1)), \mu_2+\frac{\Delta}{2})), T\}}  + \sum_{k=1}^{T}\EE{\min\{X(k, s(k), \mu_2+\frac{\Delta}{2}), T\}}
\end{align*}
The first inequality holds because the fact that posterior distribution of arm 1 changes according to how many rewards have arrived and the time until the sample exceeds a threshold is bounded by the maximum of using fixed posterior. The second inequality holds because the maximum is bounded by the sum. The third inequality holds because each index $k$ enters exactly once except the starting index $n_{t_j}$ which could have repetitions. 
\end{proof}

\section{Proof of Theorem \ref{thm:n-arm}} \label{proof:thm3.6}

\begin{proof}
As in \cite{pmlr-v23-agrawal12}, we define arm $i\ne 1$ to be saturated if the number of observed rewards is at least $\frac{24 \log T}{\Delta_i^2}$ and let the set of saturated arms at time $t$ be $C(t)$. Define the following two events:
\[
F_1 = \left\{ \exists t,i: m_t(i) \geq \frac{24\log(T)}{q_i}, 
    n_{t+d_i(q_i)}(i) < \frac{q_i}{2}m_t(i)\right\}
\]
\[
F_2 = \left\{\exists t \le T, i \in C(t): \theta_i(t) \notin [\mu_i - \frac{\Delta_i}{2}, \mu_i + \frac{\Delta_i}{2}]\right\}
\]
Then we have the following results adapting proofs from \citet{MAB_unrestricted_delay} and \citet{pmlr-v23-agrawal12} respectively as in the 2-arm case. 
\begin{lemm}
$\mathbb{P}(F_1) \leq \frac{1}{T}$ and $\mathbb{P}(F_2) \le \frac{4(K-1)}{T}$.
\end{lemm}
Let $G = \neg F_1 \bigcap \neg F_2$. Similar to the two arm case we focus on the clean event $G$. As in 2-arm case we only need to bound the regret condition on $G$. We ignore condition on $G$ for simplicty in the following proof. Define 
\[
\tau_i = \inf\{t : n_t(i) \geq 24 \Delta_i^{-2} \log(T)\}
\]
which means the first time arm $i$ is in the saturated set. We bound the regret in terms of playing saturated arms and non-saturated arms. To bound the regret due to non-saturated arms, note that condition $G$, we have 
\begin{lemm}
$m_{\tau_i}(i) \leq \frac{48\log T}{q_i\Delta_i^2} + d_i(q_i)$.
\end{lemm}
\begin{proof}
If $m_{\tau_i-d_i(q_i)}(i) \leq \frac{24\log T}{q_i}$ then the conclusion obviously holds. If not since we are in $G$, we know 
\[
n_{\tau_i - d_i(q_i)+d_i(q_i)}(i) \geq \frac{q_i}{2}m_{\tau_i-d_i(q_i)}(i)
\]
which implies 
\[
m_{\tau_i-d_i(q_i)}(i) \leq \frac{48\log T}{q_i\Delta_i^2}.
\]
Hence, 
\[
m_{\tau_i}(i) \leq \frac{48\log T}{q_i\Delta_i^2} + d_i(q_i)
\]
\end{proof}
Hence the regret due to unsaturated arms is bounded by 
\[
\sum_{i \ne 1} \frac{48\log T}{q_i\Delta_i} + d_i(q_i)\Delta_i.
\]
To bound the regret due to playing saturated arms, we follow \cite{pmlr-v23-agrawal12} and incorporate delays into the arguments. Specifically let us use the notations they developed. Let $I_j$ denote the interval between (excluding) $t_j$ and $t_{j+1}$. Define the following event 
\[
M(t) = \{\theta_1(t) > \max_{i \in C(t)} (\mu_i + \frac{\Delta_i}{2})\}
\]
and assume $M(t)$ holds if $C(t)$ is empty. Note that under $G$ all saturated arms are concentrated so essentially $M(t)$ denotes a pull of unsaturated arm. Now let $\gamma_j = |\{t\in I_j: M(t) = 1\}|$ and let $I_j(l)$ denotes the sub-interval of $I_j$ between $(l-1)$-th and $l$-th occurrences of event $M(t)$ in $I_j$. Finally let 
\[
V_j^{l,a} = |\{t \in I_j(l): \mu_a = \max_{i \in C(t)}\mu_i\}|
\]
which segments the interval $I_j(l)$ by which saturated arm to pull. Let $\mathcal{R}^S(I_j)$ be the regret pulling saturated arms in interval $I_j$. We then have the following crucial lemma bounding regret due to playing saturated arms from \citet{pmlr-v23-agrawal12}.
\begin{lemm} \label{lemm:saturated}
$\sum_{j=0}^{T-1}\EE{\mathcal{R}^s(I_j)} \le \sum_{j=0}^{T-1}\EE{\sum_{\ell=1}^{\gamma_j+1} \sum_{a=2}^K 3\Delta_a V^{\ell,a}_j} + 4(K-1)$.
\end{lemm}
We then have 
\begin{align}
    &\sum_{j=0}^{T-1}\EE{\sum_{\ell=1}^{\gamma_j+1} \sum_{a=2}^K 3\Delta_a V^{\ell,a}_j} \notag \\
    =& \sum_{a=2}^K 3\Delta_a\EE{\sum_{j=0}^{T-1}\sum_{\ell=1}^{\gamma_j+1} V^{\ell,a}_j} \label{target}
\end{align}
Now as in non-delay case, by definition, $V_j^{l,a}$ is the number of of steps in $I_j(l)$ for which $a$ is the best arm in saturated set and $M(t)$ does not hold. This is the steps until our Beta posterior of arm 1 has a sample exceeding $\mu_a + \frac{\Delta_a}{2}$ or an arm different than $a$ becomes the best or we reach end of round $T$. Hence this is stochastically dominated by steps until the Beta posterior sample exceeding $\mu_a + \frac{\Delta_a}{2}$. Unlike non-delayed case, the posterior distribution is changing, but it could still be bounded as the following because $V_j^{\ell,a}$ is bounded by a sum of all $X$ terms using all possible posterior distributions, i.e. 
\[
V_j^{\ell,a} \leq \sum_{k=n_{t_j}(1)}^{n_{t_{j+1}}(1)} X(k, s(k), \mu_a+\frac{\Delta_a}{2})
\]
Hence, we have 
\begin{align*}
    \eqref{target} \le & \sum_{a=2}^K 3\Delta_a\EE{\sum_{j=1}^{T}(\gamma_j+1) \sum_{k=n_{t_j}(1)}^{n_{t_{j+1}}(1)} X(k, s(k), \mu_a+\frac{\Delta_a}{2})} \\
    \le& \sum_{a=2}^K 3\Delta_a\EE{\sum_{j=1}^{T}\sum_{k=n_{t_j}(1)}^{n_{t_{j+1}}(1)} X(k, s(k), \mu_a+\frac{\Delta_a}{2})} \\
    &+ \sum_{a=2}^K 3\Delta_a\EE{\sum_{j=1}^{T}\gamma_j\max_j\sum_{k=n_{t_j}(1)}^{n_{t_{j+1}}(1)} X(k, s(k), \mu_a+\frac{\Delta_a}{2})} 
\end{align*}
Note that $\sum_{j=1}^{T}\gamma_j$ is bounded by the total number of pulls of unsaturated arm, which is $\sum_{i=2}^{K}\frac{48\log T}{q_i\Delta_i^2} + d_i(q_i)$. Also note that the maximum term is bounded by the sum over all $k$, hence the above is further bounded by 
\begin{align*}
    \le& \sum_{a=2}^K 3\Delta_a\EE{\sum_{j=1}^{T}\sum_{k=n_{t_j}(1)}^{n_{t_{j+1}(1)}} X(k, s(k), \mu_a+\frac{\Delta_a}{2})} \\
    &+ \left(\sum_{i=2}^{K}\frac{48\log T}{q_i\Delta_i^2} + d_i(q_i)\right) \sum_{a=2}^K 3\Delta_a\EE{\sum_{k=0}^T X(k, s(k), \mu_a+\frac{\Delta_a}{2})} 
\end{align*}
By the proof of lemma \ref{lemm:bound_Yj} the first term is bounded by 
\begin{equation*}
    \sum_{a=2}^K 3\Delta_a\left(\sum_{j=1}^{T}\EE{\min\{X(n_{t_j}(1), s(n_{t_j}(1)), \mu_a+\frac{\Delta_a}{2})), T\}}
+ \sum_{k=1}^{T}\EE{\min\{X(k, s(k), \mu_a+\frac{\Delta_a}{2}), T\}}\right)
\end{equation*}
which is further bounded by the following using the result from 2-arm case
\begin{equation*}
    \sum_{a=2}^{K}O\left(\frac{1}{\Delta_a} + \frac{1}{\Delta_a^3}\right) 
    + \left(\frac{32\log T}{q_1\Delta_a} + d_1(q_1)\Delta_a+\Delta_a \right)\frac{6}{\Delta_a}
\end{equation*}
The second term according to corollary \ref{coro:sum_Xk} is bounded by 
\begin{equation*}
    \left(\sum_{i=2}^{K}\frac{48\log T}{q_i\Delta_i^2} + d_i(q_i)\right) \sum_{a=2}^{K}O\left(\frac{1}{\Delta_a} + \frac{1}{\Delta_a^3}\right) 
\end{equation*}
Combining all terms gives the resulting bound. 
\end{proof}
\end{document}